\documentclass[runningheads]{llncs}
\usepackage{epsfig}
\usepackage{algorithm,algorithmic}
\usepackage{subfigure}
\usepackage{amsfonts}
\usepackage{multirow,array,bm}
\usepackage{graphicx}
\usepackage{amsmath,amssymb} 
\usepackage{color}
\usepackage[width=122mm,left=12mm,paperwidth=146mm,height=193mm,top=12mm,paperheight=217mm]{geometry}
\usepackage[colorlinks=false]{hyperref}
\begin{document}
\pagestyle{headings}
\mainmatter
\def\ECCV16SubNumber{545}  

\title{ALMN: Deep Embedding Learning with Geometrical Virtual Point Generating} 

\titlerunning{ALMN: Deep Embedding Learning with Geometrical Virtual Point Generating}

\authorrunning{Binghui Chen, Weihong Deng}

\author{Binghui Chen, Weihong Deng}
\institute{Beijing University of Posts and Telecommunications
\email{ \{chenbinghui,whdeng\}@bupt.edu.cn}}

\maketitle

\begin{abstract}
Deep embedding learning becomes more attractive for discriminative feature learning, but many methods still require hard-class mining, which is computationally complex and performance-sensitive. To this end, we propose Adaptive Large Margin N-Pair loss (ALMN) to address the aforementioned issues. Instead of exploring hard example-mining strategy, we introduce the concept of large margin constraint. This constraint aims at encouraging local-adaptive large angular decision margin among dissimilar samples in multimodal feature space so as to significantly encourage intraclass compactness and interclass separability. And it is mainly achieved by a simple yet novel geometrical Virtual Point Generating (VPG) method, which converts artificially setting a fixed margin into automatically generating a boundary training sample in feature space and is an open question. We demonstrate the effectiveness of our method on several popular datasets for image retrieval and clustering tasks.
\end{abstract}
\vspace{-2.5em}
\section{Introduction}
\vspace{-0.5em}
With the progress of deep learning \cite{krizhevsky2012imagenet,simonyan2014very,Szegedy2014Going}, deep embedding learning has received a lot of attention and has been applied in a wide range of tasks and applications, including image retrieval and clustering \cite{Arandjelovic2016NetVLAD,oh2016deep,Hershey2015Deep,hoffer2015deep}, pattern verification \cite{sun2014deep,Schroff2015FaceNet,Parkhi2015Deep,Yi2014Deep} and domain adaptation \cite{Tahmoresnezhad2016Visual,Long2014Transfer}. Deep embedding learning intends to learn a feature representation of the input image that preserves the distance between similar data points small and dissimilar data points large in the feature space.

In deep embedding learning community, most remarkable works are based on contrastive loss \cite{sun2014deep,Lin2015DeepHash,Simo2015Discriminative,Yi2014Deep} and triplet loss \cite{Schroff2015FaceNet,oh2016deep,hoffer2015deep,Parkhi2015Deep}. And it is a common knowledge that hard example mining is crucial to ensure the quality and efficiency of these above methods, since the overly easy examples can satisfy the constraint well and then produce nearly zero loss, without contributing to the parameter update during back-propagation. Nevertheless, many hard example mining methods require much computational cost when measuring the embedding vectors in feature space, and they are performance-sensitive, e.g. the hard-class mining procedure in N-pair loss \cite{Sohn2016npair}. 
\begin{figure}[t]
  \centering
  \includegraphics[width=0.9\linewidth]{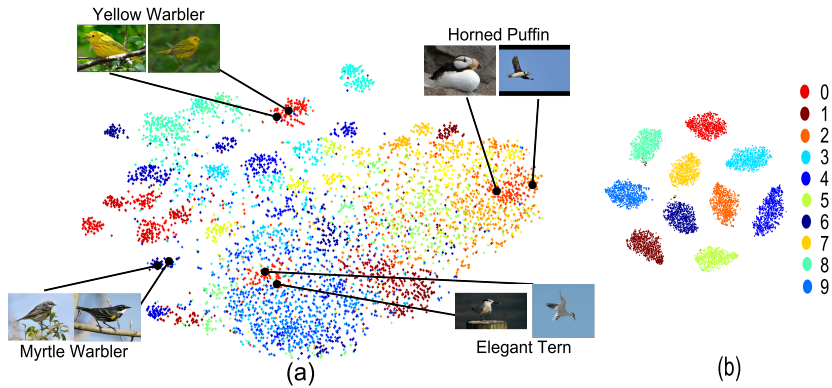}\\
  \vspace{-1em}
  \caption{Visualization (by t-SNE \cite{Laurens2015Accelerating}) of the deep embedding on the test splits of (a) CUB-200-2011 \cite{Wah2011The} (5924 images from class 101 to 200) and (b) MNIST \cite{Lecun2010The}. In (a), the
  intra-class distance can be larger than the inter-class distance, and the distribution is heterogeneous and multimodal. While in (b), the distribution is 'uniform' and ideal.}\label{fig1}
  \vspace{-1em}
\end{figure}

To alleviate the issue above and, to learn compact intra-class distance and separable inter-class distance, we introduce the concept of large margin constraint into N-pair loss instead of hard-class-mining. Some existing works \cite{Weinberger2006Distance,Liu2016Large} have focused on the learning of discriminative embedding via injecting large margin constraints into KNN and softmax, respectively. However, they exert non-adaptive constraint on the objective loss by introducing a fixed margin which is not suitable for the heterogeneous and multimodal feature distribution.

Figure \ref{fig1} illustrates the comparison between feature distribution on fine-grained bird dataset \cite{Wah2011The} and MNIST dataset \cite{Lecun2010The}. It is obviously observed that the diversity of embedding representation on bird dataset is prominent, where the intra-class distance can be larger than the inter-class distance and the distribution is heterogeneous, different from the 'uniform' distribution in MNIST. And in real cases, the distribution of feature space is complex due to pose and appearance \cite{Huang2016Local}. Thus, a consequent problem is that 
stronger margin constraint can be used for easy patterns while it is infeasible to hard patterns\footnote{Easy/hard patterns refer to where the intra-class distance is smaller/larger than the inter-class distance.}. And that is why coarsely imposing fixed constraint can not only be hard to improve the performances, but probably lead to the failure of training. Thus, introducing a prudent and local-adaptive margin constraint is of the essence.

In this paper, we propose Adaptive Large Margin N-pair loss (ALMN) to address the aforementioned issues, producing discriminative embedding under heterogeneous feature distribution in multimodal cases. It is mainly achieved by introducing an adaptive margin constraint in terms of local embedding representation structure. And as an extension to N-pair loss \cite{Sohn2016npair}, our method optimize the angular distance between samples as well. which is rotation-invariant and scale-invariant by nature. Furthermore, the adaptive large margin constraint is tactfully constructed by a novel technique of \textbf{\emph{Virtual Point Generating}} (VPG), factitiously mapping a well learned positive data point to a far place. Then, by optimizing this virtually generated new point well, a large angular margin can be obtained. Moreover, the strength of margin constraint induced by VPG for individual pattern is adjustable, quantified by hyper-parameter $\beta$. With bigger $\beta$, the ideal margin between samples becomes larger. Our ALMN is a flexible learning objective, and can be easily used as a drop-in loss function in the end-to-end frameworks and combined with any other hard example mining strategies. To our best knowledge, it is the first work to introduce margin constraint by generating virtual data point for deep embedding learning, virtual point generating is also an open question, in this work, we simply consider a geometrical way. Image retrieval and clustering experiments have been performed on several datasets, including CUB-200-2011 \cite{Wah2011The}, CARS196 \cite{Krause20133D}, Flowers102 \cite{Nilsback08}, Aircraft \cite{Maji2013Fine} and Stanford Online Products \cite{oh2016deep}.
\vspace{-1em}
\section{Related Work}
\vspace{-1em}

The very key goal of deep embedding learning is to learn a feature representation that keeps the distance between related data points small and unrelated data points large on the feature space. Some research works jointly optimize contrastive loss and softmax loss for the purpose of discriminative feature learning, such as DeepID2\cite{sun2014deep} and DeepID2+\cite{Sun2014Deeply}. Facenet \cite{Schroff2015FaceNet} proposes triplet loss to improve the ability of deep embedding learning without jointly training with softmax loss. And many remarkable works use triplet-based objective loss to optimize deep frameworks in many tasks \cite{Parkhi2015Deep,qian2015fine,hoffer2015deep,oh2016deep}. Lifted structure embedding \cite{oh2016deep} encourages that each positive pair compares the distances against all the negative pairs in one mini-batch, aiming to make full use of the mini-batch. To avoid the convergence at bad local optimum, it optimizes a smooth upper bound function of nested \emph{max} functions. Local Similarity-Aware \cite{Huang2016Local} generalizes triplet loss to a quadruplet-like loss and selects hard samples by PDDM units. N-pair loss \cite{Sohn2016npair} expands the idea of triplet or quadruplet tuple to N-pair tuple, and enforces softmax cross-entropy loss among the pairwise similarity values in the batch. We share the similar core with N-pair that takes all negative samples in the current mini-batch into consideration, but as an extension, our ALMN can lead more discriminative embedding even without hard-sample mining, as a consequence of adaptive large margin learning.

The performances of most aforementioned research works are sensitive to the selected example pairs. Selecting genius hard samples to construct a training batch can significantly improve the quality of learning, but it also incurs much computational cost. However, our ALMN does not require hard-class mining (adopted in original N-pair loss), and thus allows the training of discriminative embedding with a lower computational cost.

There are some other works aiming at learning discriminative embedding feature. Large Margin Nearest Neighbor (LMNN) \cite{Weinberger2006Distance} optimizes the Mahalanobis metric for nearest neighbor classification. Recently, Large Margin Softmax (L-Softmax) \cite{Liu2016Large} encourages the angular decision margin between classes. While, it is designed for Softmax and the margin constraints are the same for any patterns, e.g. double-angle constraint for both easy and hard patterns, thus maybe unsuitable for multimodal feature space, and the convergence of model is slow. Our ALMN allows local-adaptive margin constraint and can be successfully applied in multimodal cases.

And some other works emerge in the deep embedding learning community. Clustering\cite{songCVPR17} formulates the NMI as the objective function and optimizes it in deep models. HDC\cite{Yuan_2017_ICCV} employs the cascaded models and selects hard-samples from different levels and models. Smart-mining\cite{kumar2017smart} combines local triplet loss and global loss to optimize the deep metric with hard-samples mining. Sampling-Matters\cite{Wu_2017_ICCV} proposes distance weighted sampling strategy and use a much stronger deep model(Res-50) than most existing methods. Angular loss\cite{wang2017deep} optimize a triangle-based angular function. BIER loss\cite{Opitz_2017_ICCV} adopts ensemble learning framework of online gradients boosting which is totally different from our method that belongs to single feature learning family. Proxy-NCA\cite{Movshovitz-Attias_2017_ICCV} explains why popular classification loss works from proxy-agent view, and the implementation is very similar with Softmax. In summary, different from the above methods that investigate ways of informative samples mining or feature ensemble, we mainly focus on introducing an open question, i.e. VPG, to impose large margin constraint so as to improve the discrimination of deep embedding leaning.
\vspace{-1em}
\section{Adaptive Large Margin N-pair Loss}
\vspace{-1em}
In deep embedding learning, our goal is to learn a deep feature embedding $f(X)$ from input image $X$ into a feature vector $x\in\mathbb{R}^{d}$, such that the similarity $S(x_{i},x_{j})$ between $x_{i}$ and $x_{j}$ is higher when they belong to the same class and is lower when they belong to different classes, where $x_{*}$ refers to the feature vector of image $X_{*}$. To ensure the intra-class compactness and the inter-class separability, we introduce large margin constraint instead of exploring sample-mining strategy. One related work L-Softmax \cite{Liu2016Large} uses a preset and fixed angular margin constraint to enlarge the margin between classes. While in practical vision tasks, the embedding distribution always exhibits a character of multimode due to pose and appearance \cite{Huang2016Local}, therefore, a fixed margin constraint is not suitable. Specifically, a relatively weaker constraint will contribute little to the optimization of easy patterns, while a rigorous constraint might be too strong to guide the training of hard patterns. Under multimodal situation, the learning of discriminative feature embedding by injecting an applicable margin constraint could suit the remedy to the case. Therefore, we propose Adaptive Large Margin N-pair loss (ALMN) that can meet the needs of multimodal feature distribution. Below, we first give a review of N-pair loss, then introduce our basic objective function, and finally show the mainstay of ALMN, i.e. \emph{Virtual Point Generating}.
\vspace{-1em}
\subsection{Review of N-pair Loss and Preliminaries}
N-pair loss \cite{Sohn2016npair} points out that simultaneously optimizing with multiple negative samples can be regarded as an approximation of 'global optimization' and thus can improve the performances. It is formulated as follows:
\vspace{-0.7em}
\begin{equation}\label{eq5}
  L=-\frac{1}{N}\sum_{i}\log{\frac{e^{x_{i}^{T}x_{i^{+}}}}{e^{x_{i}^{T}x_{i^{+}}}+\sum_{y_{j}\neq{y_{i^{+}}}}e^{x_{j}^{T}x_{i^{+}}}}}+\frac{\lambda}{2N}\sum_{i=1}^{N}\|x_{i}\|_{2}^{2}
\end{equation}
where $\lambda$ is a regularization constant for $L_{2}$ norm and $N$ is the mini-batch size. $x_{i}, x_{i^{+}}, x_{j}$ refer to the positive point, anchor point and negative points respectively. Moreover, when minimizing Eq. \ref{eq5}, the optimization of inner-product-based softmax-like function is implicit to optimize the angle between samples, since the similarity based on inner product can be rewritten into $S(x_{i},x_{j})=x_{i}^{T}x_{j}=\|x_{i}\|\|x_{j}\|\cos(\theta)$, and in order to correctly separate $x_{i}$ from $x_{j}$, N-pair loss is to force $x_{i}^{T}x_{i^{+}}>x_{j}^{T}x_{i^{+}},~\forall{y_{j}\neq{y_{i}}}$, i.e. $\|x_{i}\|\cos\theta_{i}>\|x_{j}\|\cos\theta_{j}$, where $\theta_{i}/\theta_{j}$ is the angle between $x_{i}/x_{j}$ and $x_{i^{+}}$, and this optimization is mainly determined by $\cos(\theta_{.})$, verified by L-Softmax \cite{Liu2016Large}.
\vspace{-0.7em}
\subsection{Basic Objective Function based on Centers}
\vspace{-0.6em}
From minimizing Eq. \ref{eq5}, it can be observed that we would like to force $x_{i}^{T}x_{i^{+}}>x_{j}^{T}x_{i^{+}},~\forall{y_{j}\neq{y_{i}}}$ (i.e. $S(x_{i},x_{i^{+}})>S(x_{j},x_{i^{+}})$) in order to correctly separate $x_{i}$ from $x_{j}$, in another word we intend to push $x_{i}$ close to $x_{i^{+}}$ and pull $x_{j}$ far from $x_{i^{+}}$. 
Apparently, the reasonability of location of the anchor point $x_{i^{+}}$ determines the stability of model training, since anchor point $x_{i^{+}}$ affects the gradients direction, and unstable direction will impede the stability of model training. To this end, we adopt class center $c_{y_{i}}$ instead of random positive sample as our anchor point $x_{i^{+}}$. While,  it is impossible to update the class centers with respect to the entire training set during each iteration. We share a similar idea with \cite{wen2016discriminative} that performs the update on the basis of mini-batch. At each iteration, the class centers are updated as follows:
\vspace{-1em}
\begin{equation}\label{eq6}
  c_{z}^{t+1}=c_{z}^{t}-\alpha\frac{\sum_{i=1}^{N}\textbf{1}\{y_{i}=z\}\cdot(c_{z}^{t}-x_{i})}{1+\sum_{i=1}^{N}\textbf{1}\{y_{i}=z\}}
\vspace{-0.5em}
\end{equation}
where $\textbf{1}\{\emph{condition}\}=1$ if the $\emph{condition}$ is satisfied, and $\textbf{1}\{\emph{condition}\}=0$ if not. $\alpha$ is the learning rate. Finally, our basic objective loss is as follows:
\vspace{-0.7em}
\begin{small}
\begin{gather}\label{eq13}
L=-\frac{1}{N}\sum_{i}\log\frac{e^{x_{i}^{T}c_{y_{i}}}}{e^{x_{i}^{T}c_{y_{i}}}+\sum_{y_{j}\neq{y_{i}}}e^{x_{j}^{T}c_{y_{i}}}}+\frac{\lambda}{2N}\sum_{i=1}^{N}\|x_{i}\|_{2}^{2}
\end{gather}
\end{small}
\vspace{-2.5em}
\subsection{Virtual Point Generating}\label{sec3_1}
\vspace{-0.7em}
However, without hard-class mining, the constraint $x_{i}^{T}c_{y_{i}}>x_{j}^{T}c_{y_{i}},~\forall{y_{j}\neq{y_{i}}}$ \footnote{For simplicity, here, we consider the problem of binary class, where label $y\in [1,2]$. Multi-classification complicates our analysis but has the same mechanism as binary scenario.} can hardly satisfy our demands of discriminative embedding learning, since it can be easily satisfied and hence stop contributing to parameter update, as shown in Fig.\ref{fig3}.(a) where the decision boundaries for two classes are overlapped, yielding separable but not discriminative features. Inspired by L-Softmax \cite{Liu2016Large}, optimizing a rigorous objective is to produce more rigorous decision boundaries and larger decision margin, we propose \emph{Virtual Point Generating} (VPG) to enhance the constraint by generating virtually local-hard point $x_{g}$, this constraint based on the generated point is more suitable in multimodal space than L-Softmax, producing an adaptive decision margin. Here, we will first introduce the general concept of our VPG and then will explain how to make it adaptive. Since the training of Eq. \ref{eq13} is based on angular optimization, $x_{g}$ is thus generated in the angular manner, and to keep the same amplitude as $x_{i}$, we formulate $x_{g}$ as follows:
\vspace{-0.5em}
\begin{equation}\label{eq2}
  x_{g}=\frac{\beta b+x_{i}}{\|\beta b+x_{i}\|}\|x_{i}\|
  \vspace{-0.5em}
\end{equation}
As shown in Fig.\ref{fig3}, vector $b$ has the same direction with $x_{i}-c_{y_{i}}$ and affects the location of $x_{g}$, here we do not focus on its specific value which will be investigated later. $\beta$ is a hyper-parameter to further control the location of $x_{g}$. From the right chart of Figure \ref{fig3} ($\beta=1$), it can be observed that the new generated data point $x_{g}$ has a larger angular distance to the anchor point $c_{y_{i}}$ than $x_{i}$. Therefore, to make a more rigorous decision boundary, we instead require
\vspace{-0.5em}
\begin{equation}\label{eq1}
 x_{g}^{T}c_{y_{i}}>x_{j}^{T}c_{y_{i}},~\forall{y_{j}\neq{y_{i}}}
 \vspace{-0.5em}
\end{equation}
Due to the geometrical relationship in Figure \ref{fig3}, $x_{i}^{T}c_{y_{i}}>x_{g}^{T}c_{y_{i}}$ always hold, if we can optimize $x_{g}^{T}c_{y_{i}}>x_{j}^{T}c_{y_{i}}$, then $x_{i}^{T}c_{y_{i}}>x_{j}^{T}c_{y_{i}}$ will spontaneously hold. So the new objective (i.e. Eq.\ref{eq1}) is a stronger constraint (requirement) to correctly separate $x_{i}$ from $x_{j}$, producing more rigorous decision boundaries.
\begin{figure}[t]
  \centering
  \includegraphics[width=0.82\linewidth]{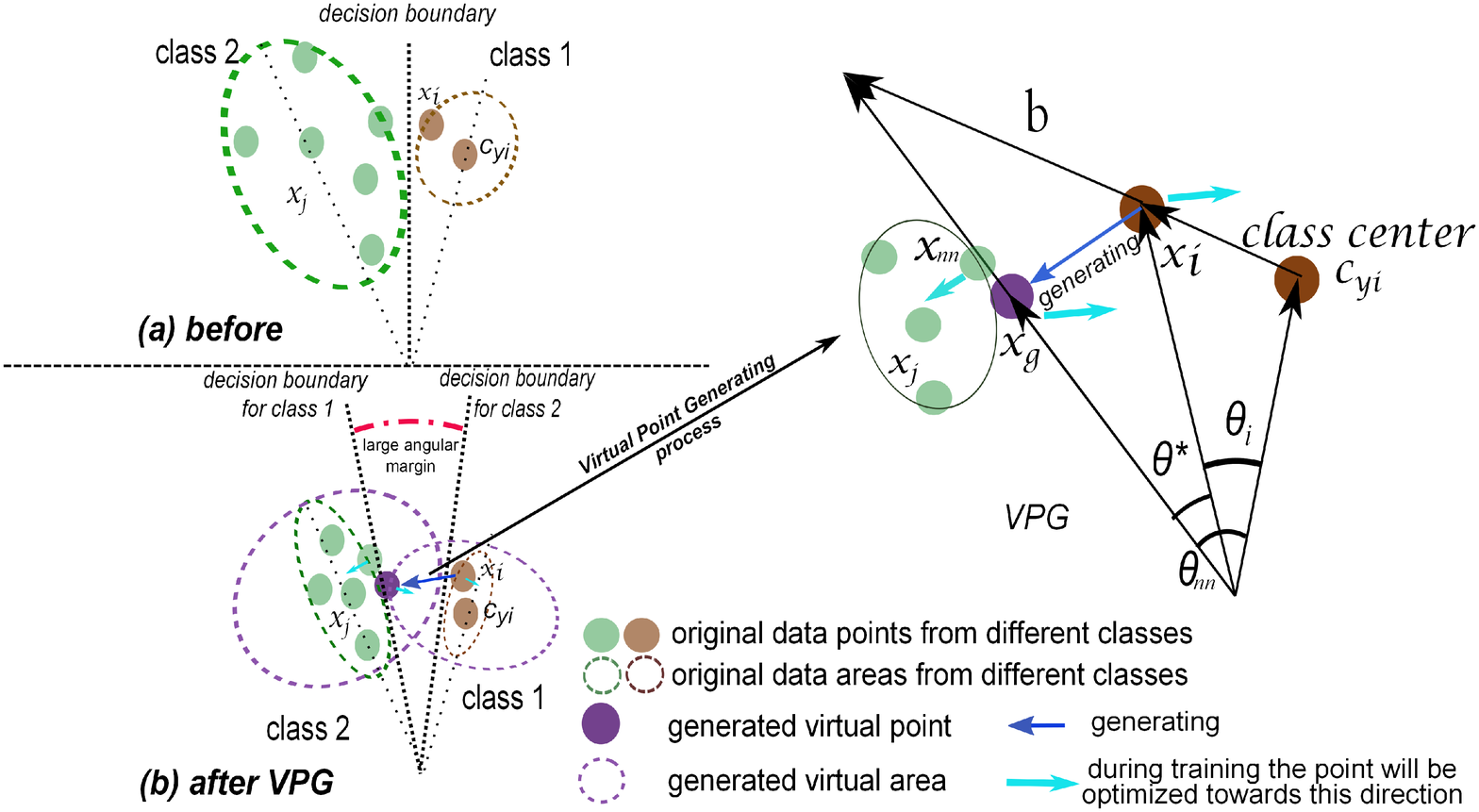}\\
  \vspace{-1em}
  \caption{Geometric interpretation of VPG ($\beta=1$). The embedding features learned before and after VPG are shown in left chart, one can observe that the angular margin between brown and green classes is enlarged by VPG, since the generated purple point is the boundary example and optimizing it will benefit the discriminative feature learning. The generating process is shown in right chart.
  }\label{fig3}
  \vspace{-2em}
\end{figure}

As illustrated in the left chart of Figure \ref{fig3}, optimizing the objective $x_{g}^{T}c_{y_{i}}>x_{j}^{T}c_{y_{i}}$, which is implicitly with a stronger margin constraint, is to produce a large angular decision margin between classes, and to encourage both intra-class compactness and inter-class separability. Specifically, as in Fig.\ref{fig3}.(a) before VPG, when the training loss get to a stable level, the data points in feature space have no need to move further because they have satisfied the constraint $x_{i}^{T}c_{y_{i}}>x_{j}^{T}c_{y_{i}}$ well, however after VPG as shown in Fig.\ref{fig3}.(b), $x_{i}$ is mapped to a boundary point or even much harder point in feature space, i.e. $x_{g}$, so as to correctly separate $x_{g}$ from $x_{j}$, the new decision boundary is produced, and it will further push $x_{g}$ as well as $x_{i}$ towards $c_{y_{i}}$ and $x_{j}$ far from $c_{y_{i}}$ in angular manner, yielding more compact intra-class and separable inter-class angular distributions. Moreover, naturally inferred from Figure \ref{fig3}, by increasing $\beta$ to a larger value (e.g. $2,3,\ldots$), a farther $x_{g}$ is generated, in another word, a more rigorous objective is to be optimized and thus in ideal case a more discriminative embedding can be achieved. 

\textbf{Adaptive Margin}: Without loss of generality, we consider $\beta=1$. As mentioned above, our goal is to make an adaptive large margin constraint, and from Eq. \ref{eq2} one can observe that $x_{g}$ is mainly determined by the vector $b$. Hence, vector $b$ should be local-adaptive such that the margin constraint based on $x_{g}$ is applicable for each case, e.g. hard and easy patterns. Specifically, considering the local feature space, vector $b$ should satisfy  $\theta^{*}=\theta_{nn}-\theta_{i}$ (as in Figure \ref{fig3}), where $\theta_{nn}$ is the angle between $c_{y_{i}}$ and its nearest negative vector $x_{nn}$, and $\theta_{i}$ is the angle between $c_{y_{i}}$ and $x_{i}$. In summary, since $x_{g}$ is based on $x_{nn}$, with considering the local feature structure, the margin constraint introduced by $x_{g}$ is adaptive, in another word, easy patterns (larger $\theta_{nn}$ and smaller $\theta_{i}$, i.e. larger $\theta^{*}$) can be equipped with relatively stronger constraint, and hard patterns (smaller $\theta_{nn}$ and larger $\theta_{i}$, i.e. smaller $\theta^{*}$) with weaker constraint. As a consequence, the margin constraint is adaptive.

To generate $x_{g}$, we need to compute the specific value of $b$. However, it is not our focus and its specific value does not matter. Since, in practical application, we adopt random sampling instead of hard negative sample mining and only one mini-batch is fed into the network each iteration, so in one mini-batch $x_{nn}$ is not globally optimal and is always much farther, resulting in a bigger $\theta_{nn}$ (bigger $\theta^{*}$), i.e. bigger $\|b\|$, in another word a farther $x_{g}$ and non-local margin constraint are introduced. As a consequence, the training will be hard and even get failure. We address this challenge by empirically and experimentally constructing a lower bound vector \footnote{The lower bound vector has the same direction with the original vector, yet smaller amplitude.} of $b$, i.e. $b_{L}$, as follows:
\vspace{-0.1em}
\begin{equation}\label{eq3}
b_{L}=\frac{x_{i}-c_{y_{i}}}{\|x_{i}-c_{y_{i}}\|}\|x_{i}\|\sqrt{2-2\cos{(\theta_{nn}-\theta_{i})}}
\vspace{-1.5em}
\end{equation}
\begin{proposition}
$b_{L}$ is a lower bound vector of $b$ as illustrated in Figure \ref{fig4}.
\end{proposition}
\begin{proof}
We provide a explicit geometric interpretation of this lower bound vector $b_{L}$. As shown in Figure \ref{fig4}.(a), since $\|x_{g}\|=\|x_{i}\|$, and according to the Cosine Law, in $\triangle{ox_{g}x_{i}}$, $\|x_{g}-x_{i}\|=\sqrt{\|x_{g}\|^{2}+\|x_{i}\|^{2}-2\|x_{g}\|\|x_{i}\|\cos{\theta^{*}}}=\|x_{i}\|\sqrt{2-2\cos{(\theta_{nn}-\theta_{i})}}$, Additionally, $x_{i}$ and $x_{g}$ are on one concentric circle and easy to prove $\theta_{2}<\theta_{4}<\frac{\pi}{2}<\theta_{3}$, according to the Sine Law, in $\triangle{x_{g}bx_{i}}$, $\frac{\|x_{g}-x_{i}\|}{\|b\|}=\frac{\sin{\theta_{2}}}{\sin{\theta_{3}}}<1$. So $\|x_{g}-x_{i}\|<\|b\|$ always holds and from Eq. \ref{eq3} we have $\|b_{L}\|=\|x_{g}-x_{i}\|$, thus $\|b_{L}\|<\|b\|$ and, vector $b_{L}$ and $b$ have the same direction with $x_{i}-c_{y_{i}}$. In conclusion, $b_{L}$ in Eq. \ref{eq3} can be regarded as a lower bound vector of $b$.
\end{proof}
\begin{figure}[h]
\vspace{-1.5em}
  \centering
  \includegraphics[width=0.8\linewidth]{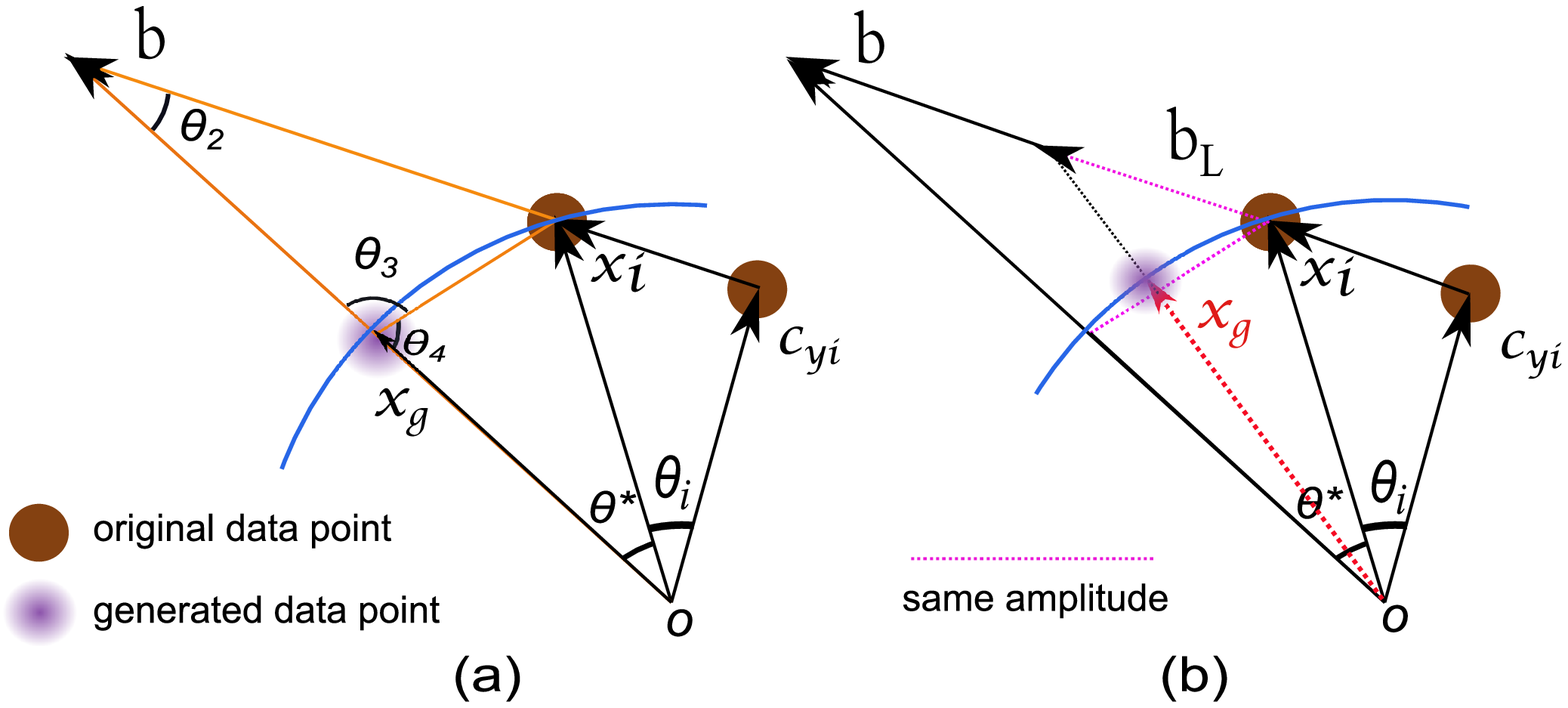}\\
  \vspace{-1em}
  \caption{(a) gives the geometric proof. (b) shows the stable $x_{g}$ generated by $b_{L}$ ($\beta=1$).}\label{fig4}
  \vspace{-2em}
\end{figure}

Replacing $b$ in Eq. \ref{eq2} with the lower bound vector $b_{L}$, we can obtain a more stable $x_{g}$ as depicted in Figure \ref{fig4}.(b) and formulate it as follows:\vspace{-0.5em}
\begin{equation}\label{eq11}
\vspace{-0.7em}
  x_{g}=\frac{\beta b_{L}+x_{i}}{\|\beta b_{L}+x_{i}\|}\|x_{i}\|=\frac{(M+1)x_{i}-Mc_{y_{i}}}{\|(M+1)x_{i}-Mc_{y_{i}}\|}\|x_{i}\|
\end{equation}
where $M=\frac{\beta\|x_{i}\|\sqrt{2-2\cos{(\theta_{nn}-\theta_{i})}}}{\|x_{i}-c_{y_{i}}\|}$, it addresses the problem of less-than-ideal angular constraint to some extent, which is caused by random sample mining and mini-batch training. We experimentally find that it indeed works well and also allows the stability of network optimizing.

\textbf{Overall Objective}: to optimize the new rigorous objective $x_{g}^{T}c_{y_{i}}>x_{j}^{T}c_{y_{i}}$, we follow N-pair loss and formulate it as the following one, i.e. our ALMN loss:
\begin{small}
\begin{gather}\label{eq7}
  L=-\frac{1}{N}\sum_{i}\log\frac{e^{x_{g}^{T}c_{y_{i}}}}{e^{x_{g}^{T}c_{y_{i}}}+\sum_{y_{j}\neq{y_{i}}}e^{x_{j}^{T}c_{y_{i}}}}+\frac{\lambda}{2N}\sum_{i=1}^{N}\|x_{i}\|_{2}^{2}
\end{gather}
\end{small}
where $x_{g}$ is shown in Eq. \ref{eq11}. Obviously when $\beta=0, x_{g}=x_{i}$, and we make it as our baseline. The ALMN can be easily optimized by commonly used SGD and BP algorithm. The gradients with respect to $x_{i}$ and $x_{j}$ are listed as follows:
\begin{gather}\label{eq14}
\frac{\partial{L}}{\partial{x_{i}}}=\frac{1}{N}(\frac{e^{x_{g}^{T}c_{y_{i}}}}{e^{x_{g}^{T}c_{y_{i}}}+\sum_{y_{j}\neq{y_{i}}}e^{x_{j}^{T}c_{y_{i}}}}-1)\frac{\partial{(x_{g}^{T}c_{y_{i}})}}{\partial{x_{i}}}+\frac{\lambda}{N}x_{i}\\
\frac{\partial{L}}{\partial{x_{j}}}=\frac{1}{N}\frac{e^{x_{j}^{T}c_{y_{i}}}}{e^{x_{g}^{T}c_{y_{i}}}+\sum_{y_{j}\neq{y_{i}}}e^{x_{j}^{T}c_{y_{i}}}}c_{y_{i}}
\vspace{-1em}
\end{gather}
\vspace{-1em}
\begin{align}\label{eq15}
\frac{\partial{(x_{g}^{T}c_{y_{i}})}}{\partial{x_{i}}}=\frac{(M+1)x_{i}^{T}c_{y_{i}}-Mc_{y_{i}}^{T}c_{y_{i}}}{\|(M+1)x_{i}-Mc_{y_{i}}\|\|x_{i}\|}x_{i}+\nonumber
\end{align}
\vspace{-1em}
\begin{align}
\frac{(M+1+\frac{M(x_{i}-c_{y_{i}})^{T}c_{y_{i}}}{\|x_{i}-c_{y_{i}}\|^{2}})\|x_{i}\|c_{y_{i}}+M\frac{(\|x_{i}-c_{y_{i}}\|^{2}-\|x_{i}\|^{2})(x_{i}-c_{y_{i}})^{T}c_{y_{i}}}{\|x_{i}\|\|x_{i}-c_{y_{i}}\|^{2}}x_{i}}{\|(M+1)x_{i}-Mc_{y_{i}}\|}\nonumber
\end{align}
\vspace{-1em}
\begin{align}
-\frac{((M+1)x_{i}-Mc_{y_{i}})^{T}c_{y_{i}}(M(x_{i}-c_{y_{i}})^{T}x_{i}+\|x_{i}\|^{2})((M+1)x_{i}-Mc_{y_{i}})}{\|(M+1)x_{i}-Mc_{y_{i}}\|^{3}\|x_{i}\|}
\end{align}
\vspace{-1em}
\begin{algorithm}
\caption{Training deep model with our ALMN}\label{algorithm}
\begin{algorithmic}[1]
\REQUIRE  training set $\{x_{i},y_{i}\}_{i=1}^{N}$ ($N$ denotes the image number), pre-trained CNN model, hyper-parameter $\beta$.
\end{algorithmic}
\textbf{Training}
\begin{algorithmic}[1]
\FOR{$t:=1\ldots T$}
 \FOR{$i:=1\ldots N$}
 \STATE adopt $c_{y_{i}}$ as the anchor point, compute $\theta_{nn}$, $\theta_{i}$
 \STATE generate $x_{g}$ from $x_{i}$ with Eq.\ref{eq11}.
 \STATE compute $Loss$ with Eq.\ref{eq7}, compute gradients with Eq. \ref{eq14}-\ref{eq15}.
 \ENDFOR
   \STATE update the anchor point with Eq.\ref{eq6}.
\ENDFOR
\end{algorithmic}
\textbf{Output:} Well trained deep model.
\end{algorithm}
\vspace{-1em}

Finally, we show ALMN in Algorithm.\ref{algorithm}. Most worthy of mention is that we introduce the novel concept of VPG to enhance the margin constraint, i.e. generating a virtually boundary point and optimizing $x_{g}^{T}c_{y_{i}}>x_{j}^{T}c_{y_{i}}$ instead of the original $x_{i}^{T}c_{y_{i}}>x_{j}^{T}c_{y_{i}}$. While our VPG does not limit the specific formulation of $x_{g}$, we leave it as an open question and there can be other ways to generate $x_{g}$, here for geometrical interpretation, we simply take Eq.\ref{eq2} and \ref{eq11}.
\section{Discussion}
\vspace{-1em}
The ALMN loss encourages an adaptive large angular margin among classes by a novel constraint constructing method VPG. 
It has some nice properties:
\begin{itemize}
\vspace{-0.5em}
   \item The core of VPG is to enhance the margin constraint by generating virtually hard points. And the holistic margin constraint can be controlled by hyper-parameter $\beta$. With bigger $\beta$, the ideal margin between classes becomes larger, yielding more discriminative embedding.
   \item For any fixed $\beta$, the angular margin constraint induced by VPG is local-adaptive and varies across instances, since the virtual point is generated on the basis of local feature structure. Thus, easy patterns can be supervised by stronger constraint, and hard patterns will be optimized under the relatively weaker constraint.
   \item Our VPG is a generic method that can be easily combined with any other hard-sample-mining methods and model architectures.
\end{itemize}

\textbf{Comparison to N-pair loss}: as an extension to N-pair loss \cite{Sohn2016npair}, our ALMN has two advantages. First, by employing class centers $c_{y_{i}}$ as the anchor points instead of random positive points, the optimization of our ALMN is more stable and ideal than N-pair loss due to the correct direction of gradients, and thus the performance of deep embedding learning can be improved, verified by the results comparison between ALMN ($\beta=0$) and N-pair loss in Table. \ref{tab2} and \ref{tab3}. Second, and which is our most contribution, ALMN (e.g. $\beta=3$) can significantly encourage a large angular decision margin among classes, yielding more discriminative feature embedding than N-pair loss, and it is mainly achieved by the novel and generic VPG method. Furthermore, our ALMN does not require hard-class mining procedure which is adopted to construct the training batches in N-pair loss.

\textbf{Comparison to other constraint losses}: Noisy-Softmax \cite{Chen_2017_CVPR} imposes annealed noise on Softmax which aims to improve the generalization ability of DCNNs. Our ALMN has a similar goal with \cite{Liu2016Large,liu2017sphereface} that enhancing the discriminative property of learned features by exerting constraint on objective function. However, in \cite{Liu2016Large}, the constraint is specifically designed for Softmax layer, and the strength of margin constraint behind the optimization objective $\|w_{y_{i}}\|\|x_{i}\|\cos{m\theta_{y_{i}}}>\|w_{j}\|\|x_{i}\|\cos\theta_{j}$ are the same for each samples (e.g. m=2), and this fixed m-times-angle constraint is not applicable under heterogeneous feature distribution. In contrast, our ALMN is towards deep embedding learning, for a certain $\beta$, our margin constraint behind $x_{g}^{T}c_{y_{i}}>x_{j}^{T}c_{y_{i}}$ is local-adaptive, since the virtual point $x_{g}$ is generated on the basis of its neighbouring feature space not a fixed scale. And, the margin constraint of ALMN is introduced by \emph{generated virtual point} which is different from directly setting $m$ in \cite{Liu2016Large}.

\textbf{Ablation study}: to highlight the effectiveness of our local-adaptive large margin constraint, we conduct a contrast test by modifying our basic objective function (Eq. \ref{eq13}) to a L-Softmax-like loss, which is of the fixed angular margin constraint, as follows:
\begin{footnotesize}
\begin{equation}\label{eq12}
  L=-\frac{1}{N}\sum_{i}\log{\frac{e^{\|x_{i}\|\|c_{y_{i}}\|\psi{(\theta_{y_{i}})}}}{e^{\|x_{i}\|\|c_{y_{i}}\|\psi{(\theta_{y_{i}})}}+\sum_{y_{j}\neq{y_{i}}}e^{x_{j}^{T}c_{y_{i}}}}}+\frac{\lambda}{2N}\sum_{i=1}^{N}\|x_{i}\|_{2}^{2}
\end{equation}
\end{footnotesize}
where $\psi{(\theta_{y_{i}})}=(-1)^{k}\cos{(m\theta_{y_{i}})}-2k$ is the same as in L-softmax. Then, we train the same CNN model with Eq. \ref{eq12} ($m=2$) and Eq. \ref{eq7} ($\beta=2$), respectively. From Figure \ref{fig6}, we can observe that the training loss of L-Softmax($m=2$) stops reducing at a higher level, implying it does not converge, and we infer that the double-angle constraint may be much stronger for some examples (e.g. hard patterns) and this phenomenon will disturb the overall training process. While, the loss of our ALMN drops fast to a relatively low level, demonstrating that the local-adaptive angular margin constraint can be well optimized and thus is indeed crucial to address the problem of discriminative embedding learning in multimodal cases.

\begin{figure}[t]
  \vspace{-1em}
  \centering
  \begin{minipage}{0.45\linewidth}
  \centering
  \includegraphics[width=1\linewidth]{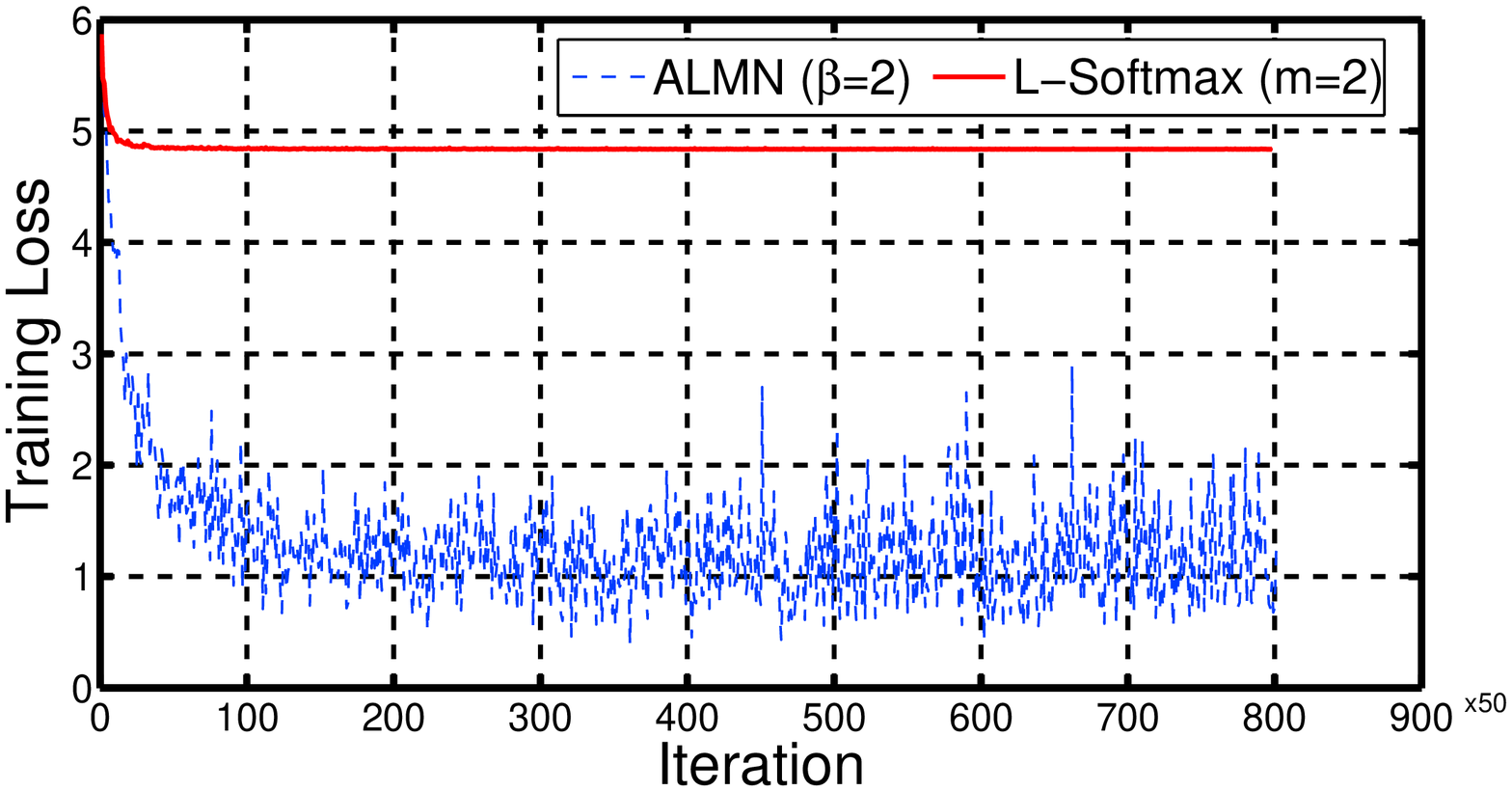}\\
  \vspace{-1em}
  \caption{Training loss on CUB-200-2011 dataset.}\label{fig6}
  \end{minipage}
  \begin{minipage}{0.45\linewidth}
  \centering
  \includegraphics[width=1\linewidth]{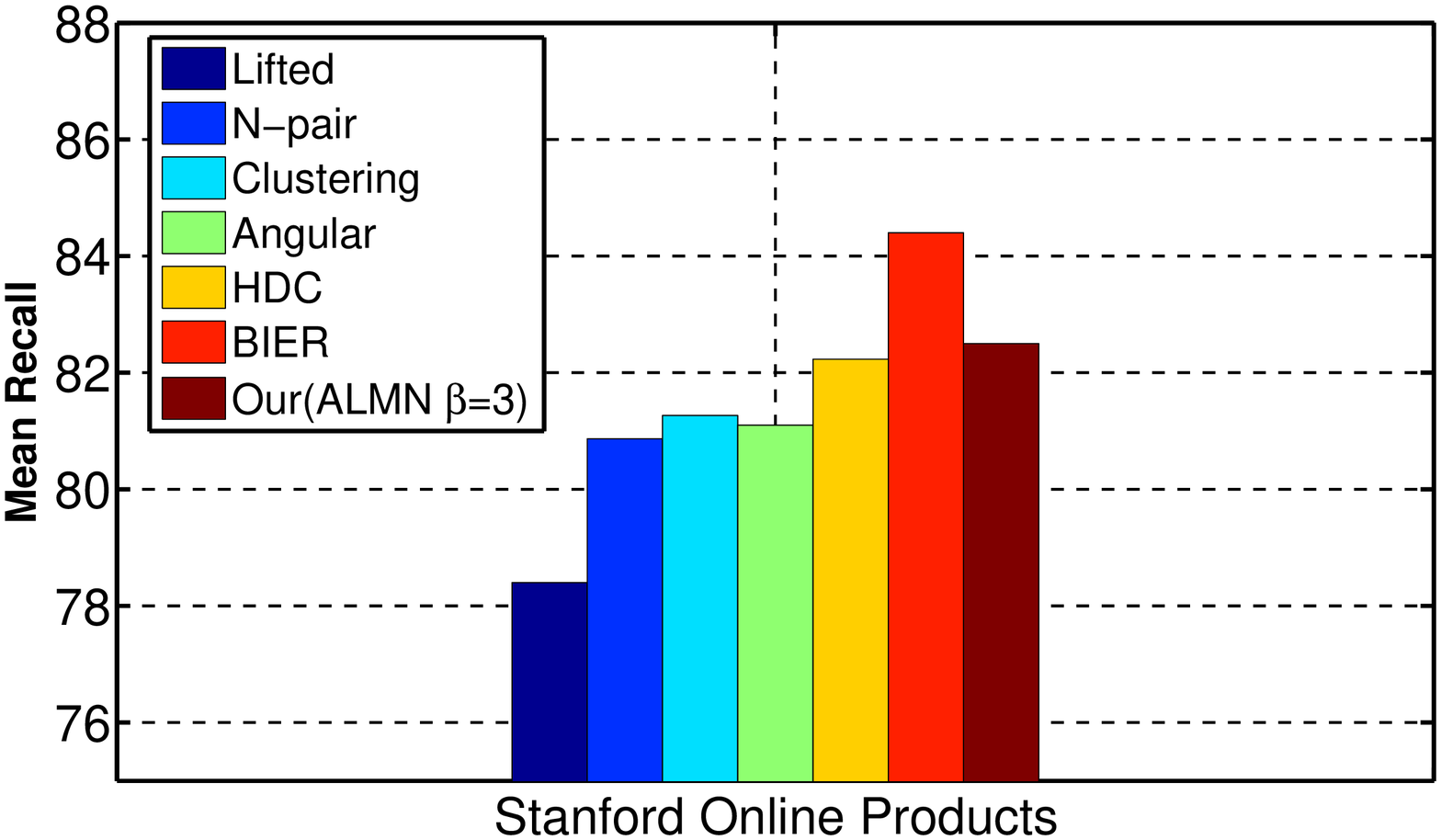}\\
    \vspace{-1em}
    \caption{Mean Recall on Stanford Online Products.}\label{fig7}
  \end{minipage}
    \begin{minipage}{0.45\linewidth}
  \centering
  \includegraphics[width=1\linewidth]{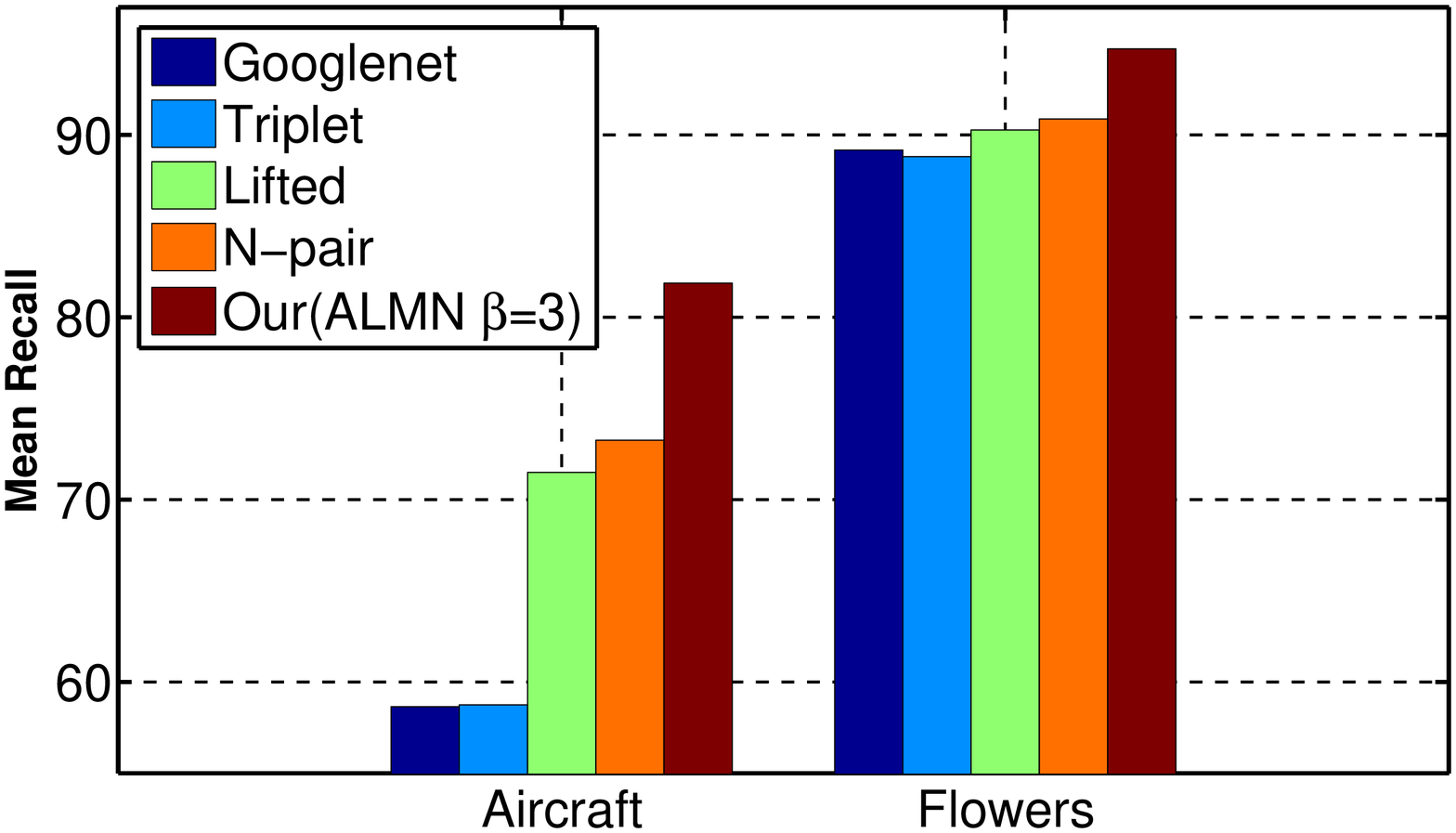}\\
    \vspace{-1em}
    \caption{Mean Recall on Aircraft and Flowers.}\label{fig8}
  \end{minipage}
    \begin{minipage}{0.45\linewidth}
  \centering
  \includegraphics[width=1\linewidth]{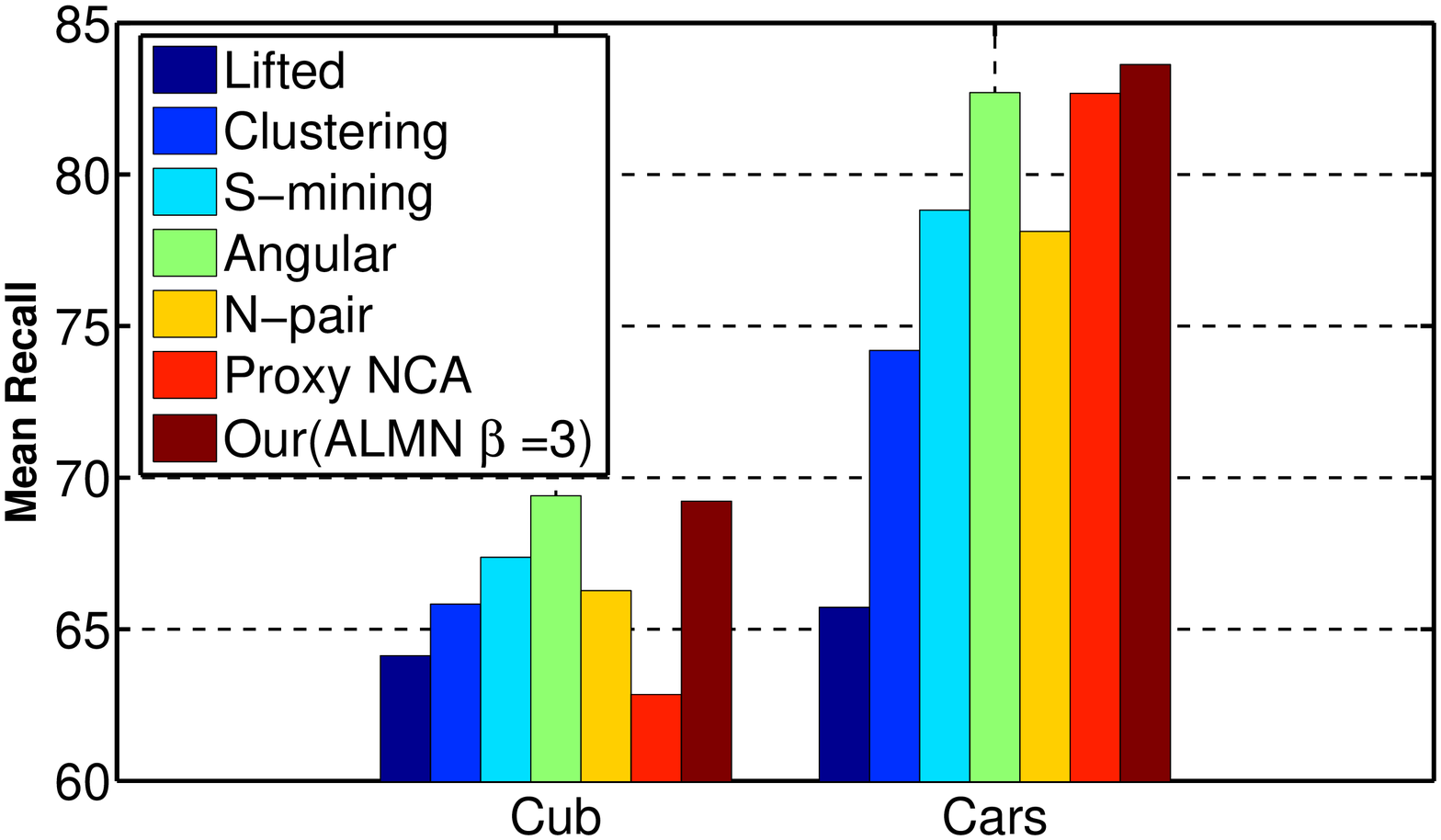}\\
    \vspace{-1em}
    \caption{Mean Recall on Cub and Cars.}\label{fig9}
  \end{minipage}
  \vspace{-1em}
\end{figure}
\vspace{-0.7em}
\section{Experiments and Results}
\vspace{-0.7em}
To demonstrate the effectiveness of our proposed ALMN under multimodal scenarios, we evaluate it on image clustering and retrieval tasks over several benchmark datasets, which present varieties of variations such as in pose and appearance. Notably, except class label we do not use any other annotation information such as bounding box or part annotation.
\vspace{-1em}
\subsection{Implementation Details}\label{sec5_1}
For network configuration, we use the ImageNet pretrained GoogLeNet \cite{Szegedy2014Going} for initialization and finetune it on our target datasets. The last fully connected layer is initialized with random weights and we fix the embedding feature size at $512$ throughout all of our experiments(since the performance doesn't change much when varying embedding sizes according to \cite{oh2016deep}). We set dropout ratio to $0.2$. For fair comparison, we follow the same data preprocess method as adopted in \cite{oh2016deep}, i.e. all the training and testing images are processed into $256\times256$ and then mean subtraction is performed. For data augmentation, all training images are randomly cropped to $227\times227$ and randomly mirrored. All of our experiments are implemented by Caffe library \cite{Jia2014Caffe} with our own modifications.

As we mentioned in the above section, we do not perform hard-class mining. Instead, we construct a random batch in $m\times n$ manner, where $m$ and $n$ denote the number of classes and the number of samples in each class, respectively. Note that, the classes and samples are all randomly selected. And we will investigate the affects of different combinations of $m$ and $n$ in the following subsection.

\textbf{Training}: The initial learning rate $\alpha$ is $0.00001$ and multiplied by $0.8$ at $20k$ iteration. However, the total iterations are $50k$ and $80k$ for (CUB, Flowers, Aircraft) and CARS196, respectively. We use a weight decay of $0.0002$ and momentum of $0.9$. Moreover, the regularization constant $\lambda$ for $L_{2}$ norm is $0.0005$ and we use 10 times learning rate for the feature layer.

\textbf{Evaluation}: The same as many other research works \cite{oh2016deep,Sohn2016npair,Huang2016Local}, we use the $F_{1}$ and NMI metrics for image clustering task and the Recall@K metric for image retrieval task. We use the simple cosine distance for the evaluation of the embedding feature. We make ALMN ($\beta=0$), which means training without VPG (i.e. $x_{g}=x_{i}$), as our baseline. For comparison, we evaluate many existing methods, and implement some of them with the same network and training configurations as ours, including triplet loss \cite{Schroff2015FaceNet}, lifted structured embedding \cite{oh2016deep} and N-pair loss \cite{Sohn2016npair}.
\vspace{-1em}
\subsection{Component Analysis}
\textbf{Mini-batch combination}:
To acquire a stable location of the anchor point, we employ the class center $c_{y_{i}}$. However, we experimentally found that the combination of mini-batch is important to the update of $c_{y_{i}}$.

Inspired by N-pair loss, we construct a $m\times n$ mini-batch, where $m$ and $n$ denote the number of classes and the number of the samples in each class, respectively. Throughout our experiments the value of $m\times n$ is fixed, and we can imagine that, as $n$ increases, there are more and more positive samples to contribute to the update of $c_{y_{i}}$ at the same time, resulting in a more stable and more real class center. However, when $n$ is large enough and $m=1$, i.e. in $0$ negative sample limit, there is no contribution from negative samples and thus the inter-class separability will not be guaranteed.
\begin{table*}[ht]
\vspace{-1em}
  \centering
    \begin{tabular}{c|cccc|cccc}
    \hline
     & \multicolumn{4}{c|}{CUB-200-2011} & \multicolumn{4}{c}{Cars196} \\
     \hline
     \hline
   $m\times n$  & 65 x 2  & 26 x 5  & 16 x 8  & 8 x 16  & 65 x 2  & 26 x 5  & 16 x 8  & 8 x 16 \\
    \hline
    Recall@K=1   & 51.1  & \textbf{52.4}  & 52.1    & 51.1  & 64.2  & \textbf{71.6}  & 69.7  & 68.8 \\
    Recall@K=2   &  63.7     & \textbf{64.8}  & 64.4   &  64.0  & 75.2    & \textbf{81.3}  & 80.5   & 79.1 \\
    Recall@K=4   &  74.5     & \textbf{75.4}  &  75.6   &  74.6  & 83.7  & \textbf{88.2}  & 88.3   & 86.4 \\
    Recall@K=8   &  83.6   & \textbf{84.3}  &   84.3  &   84.0 & 90.0  & \textbf{93.4}  &  92.8  & 92.3 \\
    \hline
    F1    &    27.2   & \textbf{28.5}  &   27.5  &28  & 24.6  & \textbf{29.4}  &   26.9    & 25.3 \\
    NMI   &    59.7   & \textbf{60.7}    &   59.6  &  60.3    & 57.9  & \textbf{62.0}  &  60.9  & 58.6 \\
    \hline
    \end{tabular}%
    \caption{F1, NMI, and recall@K scores (\%) on CUB-200-2011 \cite{Wah2011The} and CARS196 \cite{Krause20133D} datasets with different combinations of $m\times n$.}
    \vspace{-3em}
  \label{tab1}%
\end{table*}%

We evaluate the performances of the ALMN loss with different combinations of $m\times n$ on CUB-200-2011 \cite{Wah2011The} and CARS196 \cite{Krause20133D}. And the experimental results are listed in Table \ref{tab1}. From the results, one can observe that the performances are different when using various combinations of $m\times n$, where the total batch sizes are almost the same. As we analyzed above, a relatively appropriate combination of $m\times n$ is required, which is important for stable training and discriminative embedding learning. And we use the combination of $26\times5$ in the following subsections. Notably, although we need to construct the mini-batch according to some protocol, the selection is totally random and there is no computational cost since there is no demand to evaluate the embedding vectors in deep learning framework, which is different from hard-class mining in N-pair loss.

\textbf{Enlarging angular margin}: We can further enhance the angular margin constraint by increasing parameter $\beta$ such that a larger decision margin among classes can be produced and the more discriminative embedding can be achieved. From Table \ref{tab2}, when $\beta=0$ in zero constraint limit, our baseline algorithm obtain relatively lower results. Then one can observe that, when $\beta=1$ our ALMN can significantly improve nearly $2\%$ and $4\%$ R@1 accuracies over CUB and CARS datasets respectively, verifying the effectiveness of the adaptive large margin constraint induced by VPG. Afterwards, it can further improve the performances over all datasets by increasing $\beta$ e.g. $\beta=2,3$, demonstrating our initial thought that larger decision margin among classes will encourage the learning of discriminative embedding. Likewise, the improvements can also be found in other datasets as in Table.\ref{tab3} \ref{tab4}.
\vspace{-1em}
\subsection{Comparison with State-of-the-art}
\vspace{-0.6em}
\begin{table*}[t]
  \centering
  \resizebox{\textwidth}{!}{%

   \begin{tabular}{c|c|c|c|c||c|c|c|c|c|c||c|c}
    \hline
    \multirow{2}[4]{*}{} & \multicolumn{6}{c|}{CUB-200-2011}             & \multicolumn{6}{c}{Cars196} \\
\cline{2-13}          & R@1   & R@2   & R@4   & R@8   & F1    & NMI   & R@1   & R@2   & R@4   & R@8   & F1    & NMI \\
    \hline
\footnotesize{Google\cite{Szegedy2014Going}$^{+}$} & 40.8  & 53.8  & 67.0    & 78.2  & 18.0    & 51.5  & 35.5  & 47.5  & 58.9  & 71.5  & 8.6   & 37.1 \\
\footnotesize{Triplet\cite{Schroff2015FaceNet}$^{+}$} & 36.1  & 48.6  & 59.3  & 70.0    & 15.1  & 49.8  & 39.1  & 50.4  & 63.3  & 74.5  & 16.8  & 51.4 \\
\footnotesize{Lifted\cite{oh2016deep}$^{+}$} & 47.2  & 58.9  & 70.2  & 80.2  & 21.2  & 55.6  & 49.0    & 60.3  & 72.1  & 81.5  & 21.8  & 55.0 \\
\footnotesize{Clustering\cite{songCVPR17}}&48.2 & 61.4 & 71.8 & 81.9 & - & 59.2 & 58.1 & 70.6 & 80.3 & 87.8 & - & 59.4\\
\footnotesize{S-mining\cite{kumar2017smart}}& 49.8& 62.3 & 74.1 & 83.3 & - & 59.9 & 64.7 & 76.2 & 84.2 & 90.2 & - & 59.5\\
\footnotesize{Angular\cite{wang2017deep}}& 53.6 & 65.0 &75.3 &83.7 & 30.2 & 61.0 & 71.3 &80.7 &87.0 &91.8 & 31.8 & 62.4\\
\footnotesize{ N-pair\cite{Sohn2016npair}$^{+}$} & 49.1    & 61.2  & 72.7  & 82.1  & 25.9  & 58.5  & 63.6  & 74.7  & 84.1  & 90.1  & 23.9  & 57.4 \\
\footnotesize{Proxy NCA\cite{Movshovitz-Attias_2017_ICCV}}&49.2&61.9&67.9&72.4&-&59.5&73.2&82.4&86.4&88.7&-&64.9\\
    \hline\hline
    ALMN ($\beta=0$)  & 50.4  & 62.7  & 73.5    & 82.9  & 27.6    & 59.4  & 66.2  & 76.7  & 85.1  & 91.4  & 23.6    & 56.7 \\
    ALMN ($\beta=1$)  & 52.0  & 64.5  & 74.8  & 83.7  & 28.2  & 60.2    & 70.4  & 80.4  & 87.3  & 92.5  & 26.3  & 59.3 \\
    ALMN ($\beta=2$)  & 52.2  & 64.7  & 75.3  & 84.2  & 28.2  & 60.7    & 71.3  & 81.2  & 88.1  & 93.1  & 28.3  & 61.5 \\
    ALMN ($\beta=3$) & \textbf{\emph{52.4}}  & \textbf{\emph{64.8}}  & \textbf{\emph{75.4}}  & \textbf{\emph{84.3}}  & \textbf{\emph{28.5}}  & \textbf{\emph{60.7}}    & \textbf{\emph{71.6}}  & \textbf{\emph{81.3}}  & \textbf{\emph{88.2}}  & \textbf{\emph{93.4}}  & \textbf{\emph{29.4}}  & \textbf{\emph{62.0}} \\
    \hline
    \end{tabular}%
    }
      \caption{Image clustering and retrieval results(\%) on CUB \cite{Wah2011The} and Cars196 \cite{Krause20133D}. $^{+}$ refers to our re-implement. Our best results are bold-faced.}
      \label{tab2}
       \vspace{-2em}
\end{table*}%
\textbf{CUB-200-2011} dataset \cite{Wah2011The} includes 11,788 bird images coming from 200 classes. We use the first 100 classes for training (5,864 images) and the rest 100 classes for testing (5,924 images). We list our experimental results together with those of other state-of-the-art methods in Table \ref{tab2}. From the results, one can observe that our baseline ALMN ($\beta=0$) outperforms N-pair loss even without large margin constraint, demonstrating that the reasonability of location of the anchor point can not only make training stable but improve the performance. And by introducing an adaptive large angular margin constraint among classes, our ALMN ($\beta=3$) can significantly improve the performances and also outperforms most existing methods, even achieving the comparable results compared to the state-of-the-art methods, and thus verifying the effectiveness of our adaptive large margin constraint.

\textbf{CARS196} dataset \cite{Krause20133D} includes 16,185 car images coming from 196 classes. We split the first 98 classes for training (8,054 images) and the rest 98 classes for testing (8,131 images). We list our experimental results together with that of other state-of-the-art methods in Table \ref{tab2}. From the results, it can be observed that ALMN ($\beta=0$) shows the better performances than N-pair loss, demonstrating the superiority of our choice of the anchor point $c_{y_{i}}$. Then, ALMN ($\beta=3$) can significantly improve nearly $5\%$ R@1 result over the baseline ALMN and also outperforms most of the other existing methods, obtains comparable results compared to state-of-the-art, verifying the effectiveness of our method.
\begin{table*}[t]
  \centering
  \resizebox{\textwidth}{!}{%
    \begin{tabular}{c|c|c|c|c||c|c|c|c|c|c||c|c}
    \hline
    \multirow{2}[4]{*}{} & \multicolumn{6}{c|}{Flowers102}             & \multicolumn{6}{c}{Aircraft}\\
\cline{2-13}          & R@1   & R@2   & R@4   & R@8   & F1    & NMI   & R@1   & R@2   & R@4   & R@8   & F1    & NMI \\
    \hline
Googlenet$^{+}$\cite{Szegedy2014Going} &80.5&87.6&92.9&95.7&41.0&63.8&42.0&52.8&64.2&75.6&10.3&30.0\\
          Triplet$^{+}$\cite{Schroff2015FaceNet} & 80.3 & 87.2 & 92.0 &95.7  &41.3  &64.0 & 41.8 & 53.5 & 64.4 & 75.3 & 10.7 & 31.3\\
    Lifted$^{+}$\cite{oh2016deep} & 82.6  & 89.4  & 93.1  & 96.0  & 43.3 & 65.9 & 53.8  & 67.5  & 77.7  & 85.5  & 23.8  & 51.9\\
    n-pair$^{+}$\cite{Sohn2016npair} & 83.3  & 89.9  & 93.9  & 96.4 & 43.2 & 66.1 & 56.1  & 69.0    & 80.2  & 87.7  & 24.7  & 52.4\\
    \hline
    \hline
    ALMN($\beta=0$)   & 85.3 &91.4 & 94.7  & 97.2 & 53.1 & 71.5 & 63.5  & 74.2  & 83.3  & 90.0    & 25.7  & 53.3\\
    ALMN($\beta=1$)   & 88.8  & 93.1  & 95.9  & 98.1  & 56.3  & 75.7 & 67.0  & 78.1  & 86.6  & 91.3  & 29.5  & 56.2\\
    ALMN($\beta=2$)   & 89.5  & 93.8  & 96.3  & 98.0  & 56.6  & 75.9 & 67.9  & 79.3  & 87.0  & 91.8  & 30.4  & 57.2\\
    ALMN($\beta=3$)   & \emph{\textbf{90.1}}  & \emph{\textbf{94.0}}  & \emph{\textbf{96.6}}  & \emph{\textbf{98.2}}  & \emph{\textbf{57.0}}  & \emph{\textbf{76.2}} & \emph{\textbf{68.4}}  & \emph{\textbf{79.9}}  & \emph{\textbf{87.2}}  & \emph{\textbf{92.0}}  & \emph{\textbf{30.7}}  & \emph{\textbf{57.9}}\\
    \hline
    \end{tabular}%
    }
      \caption{Image clustering and retrieval results on Flowers102 \cite{Nilsback08} and Aircraft dataset\cite{Maji2013Fine}. $^{+}$ refers to our re-implement. And our best results are bold-faced.}
      \label{tab3}
       \vspace{-3em}
\end{table*}%

\textbf{Flowers102} The Flowers102 dataset \cite{Nilsback08} includes 8189 flower images from 102 classes. Each class consists of between 40 and 258 images. We split the first 51 classes for training (3493 images) and the rest 51 classes for testing (4696 images). We implement triplet loss \cite{Schroff2015FaceNet}, lifted structured embedding \cite{oh2016deep} and n-pair loss \cite{Sohn2016npair} with the same network and training configurations as ours and test them with the single crop. From the results shown in Table. \ref{tab3}, our baseline ALMN($\beta=0$) outperforms other works by adopting a stable anchor point. And ALMN($\beta=3$) can further improve the performances for image clustering and retrieval tasks by learning a discriminative embedding with adaptive large margin constraint, demonstrating the superiority of our method.

\textbf{Aircraft} The Aircraft dataset \cite{Maji2013Fine} has 100 classes of aircrafts with 10,000 images. We split the first 50 classes for training (5,000 images) and the other 50 classes for testing (5,000 images). We also implement triplet loss \cite{Schroff2015FaceNet}, lifted structured embedding \cite{oh2016deep} and n-pair loss \cite{Sohn2016npair} with the same network and training configurations as ours and then test them with the single crop. From the results shown in Table. \ref{tab3}, our baseline ALMN($\beta=0$) outperforms other works by adopting a stable anchor point. And ALMN($\beta=3$) can further improve nearly $5\%$ and $6\%$ for image retrieval and clustering (F1) tasks respectively by learning a discriminative embedding with adaptive large margin constraint.

\begin{table}[H]
\vspace{-2em}
  \centering
  \resizebox{\textwidth}{!}{
    \begin{tabular}{c||c|c|c|c|c|c||c|c}
    \hline
     & Lifted\cite{oh2016deep} & n-pair\cite{Sohn2016npair} & Clustering\cite{songCVPR17} & Angular\cite{wang2017deep} &HDC\cite{Yuan_2017_ICCV} &BIER\cite{Opitz_2017_ICCV}& ALMN($\beta=0$)& ALMN($\beta=3$) \\
    \hline
    ensemble &$\times$&$\times$&$\times$&$\times$&$\surd$&$\surd$&$\times$&$\times$\\
    \hline
    R@1   & 62.5  & 66.4  & 67    & 67.9& 69.5 &72.7 & 69.3&\textbf{\emph{69.9}} \\
    \hline
    R@10   & 80.8  & 83.2  & 83.6  & 83.2& 84.4 &86.5 & 84.5&\textbf{\emph{84.8}}\\
    \hline
    R@100 & 91.9 & 93 & 93.2 & 92.2 & 92.8& 94&92.7&\textbf{\emph{92.8}}\\
    \hline
    \end{tabular}%
    }
      \caption{Results on Stanford Online dataset\cite{oh2016deep}. Our best results are bold-faced.}\label{tab4}
      \vspace{-3.5em}
\end{table}%
\textbf{Stanford Online Products} dataset\cite{oh2016deep} has $120k$ images of $22k$ online classes and each class has $5.3$ images on average. Following the zero-shot protocol, we also split the first $11318$ classes for training and the remaining $11316$ classes for testing. We show our final results in Table.\ref{tab4}. One can observe that our method ($\beta=3$) achieves appealing results compared to other single-feature methods and the ensemble-feature method(e.g. HDC\cite{Yuan_2017_ICCV} and BIER\cite{Opitz_2017_ICCV}, ensemble is well known better than single feature).

The Mean Recall comparisons over these datasets are in Figure\ref{fig9} \ref{fig8} \ref{fig7}.
\vspace{-1em}
\subsection{Cases Study}
\vspace{-0.6em}
To show the results of discriminative embedding learning under multimodal scenario, we provide some cases over CUB-200-2011 \cite{Wah2011The} and Cars196 \cite{Krause20133D} datasets in Figure \ref{fig5}. From the comparison between top-1 positive and top-1 negative retrieval, it can be observed that the image is correctly retrieved by our algorithm. Then by introducing the adaptive large margin constraint among classes, our ALMN ($\beta=3$) can significantly increase the similarity score between the query and top-1 positive retrieval images, implying that the intra-class compactness is strengthened. And from the results of top-1 negative retrieval results, one can observe that our ALMN ($\beta=3$) can significantly reduce the similarity score between the query and top-1 negative sample, demonstrating that our method produces a more separable inter-class distance.
\begin{figure*}[t]
\vspace{-1em}
  \centering
  \includegraphics[width=1\linewidth]{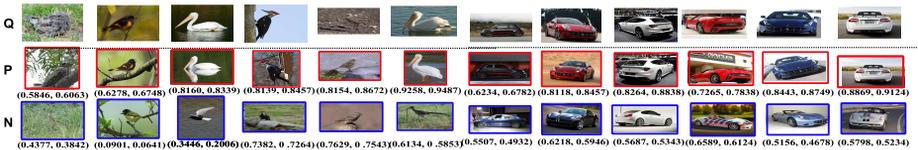}\\
  \vspace{-0.5em}
  \caption{Retrieval task cases on CUB \cite{Wah2011The} and Cars196 \cite{Krause20133D} datasets. The query images are shown on top of the figure. Top-1 positive and top-1 negative images retrieved by our ALMN are marked with red and blue boxes, respectively. And the similarity scores using ALMN ($\beta=0$) and ALMN ($\beta=3$) are orderly shown underneath the images.}\label{fig5}
  \vspace{-2em}
\end{figure*}
\vspace{-1em}
\section{Conclusion}
\vspace{-0.6em}
In this paper, we propose ALMN to address the problem of discriminating feature learning in multimodal feature space. It encourages intra-class compactness and inter-class separability by enlarging the angular decision margin among classes. And the prudent margin constraint is local-adaptive. Moreover, the novel concept of VPG gives chances of discriminative embedding learning without hard-example mining, and the virtual point generating method is an open question which may benefit the community. Extensive quantitative and qualitative results demonstrate the effectiveness of our proposed method.
\bibliographystyle{splncs}
\bibliography{egbib}

\begin{thebibliography}{10}

\bibitem{krizhevsky2012imagenet}
Krizhevsky, A., Sutskever, I., Hinton, G.E.:
\newblock Imagenet classification with deep convolutional neural networks.
\newblock In: Advances in neural information processing systems. (2012)
  1097--1105

\bibitem{simonyan2014very}
Simonyan, K., Zisserman, A.:
\newblock Very deep convolutional networks for large-scale image recognition.
\newblock arXiv preprint arXiv:1409.1556 (2014)

\bibitem{Szegedy2014Going}
Szegedy, C., Liu, W., Jia, Y., Sermanet, P., Reed, S., Anguelov, D., Erhan, D.,
  Vanhoucke, V., Rabinovich, A.:
\newblock Going deeper with convolutions.
\newblock In: Computer Vision and Pattern Recognition. (2014)  1--9

\bibitem{Arandjelovic2016NetVLAD}
Arandjelovic, R., Gronat, P., Torii, A., Pajdla, T., Sivic, J.:
\newblock Netvlad: Cnn architecture for weakly supervised place recognition.
\newblock In: IEEE Conference on Computer Vision and Pattern Recognition.
  (2016)  5297--5307

\bibitem{oh2016deep}
Oh~Song, H., Xiang, Y., Jegelka, S., Savarese, S.:
\newblock Deep metric learning via lifted structured feature embedding.
\newblock In: Proceedings of the IEEE Conference on Computer Vision and Pattern
  Recognition. (2016)  4004--4012

\bibitem{Hershey2015Deep}
Hershey, J.R., Chen, Z., Roux, J.L., Watanabe, S.:
\newblock Deep clustering: Discriminative embeddings for segmentation and
  separation.
\newblock In: IEEE International Conference on Acoustics, Speech and Signal
  Processing. (2015)  31--35

\bibitem{hoffer2015deep}
Hoffer, E., Ailon, N.:
\newblock Deep metric learning using triplet network.
\newblock In: International Workshop on Similarity-Based Pattern Recognition,
  Springer (2015)  84--92

\bibitem{sun2014deep}
Sun, Y., Chen, Y., Wang, X., Tang, X.:
\newblock Deep learning face representation by joint
  identification-verification.
\newblock In: Advances in neural information processing systems. (2014)
  1988--1996

\bibitem{Schroff2015FaceNet}
Schroff, F.e.a.:
\newblock Facenet: A unified embedding for face recognition and clustering.
\newblock In: Proceedings of the IEEE Conference on Computer Vision and Pattern
  Recognition. (2015)  815--823

\bibitem{Parkhi2015Deep}
Parkhi, O.M., Vedaldi, A., Zisserman, A.:
\newblock Deep face recognition.
\newblock In: British Machine Vision Conference. (2015)  41.1--41.12

\bibitem{Yi2014Deep}
Yi, D., Lei, Z., Liao, S., Li, S.Z.:
\newblock Deep metric learning for person re-identification.
\newblock In: International Conference on Pattern Recognition. (2014)  34--39

\bibitem{Tahmoresnezhad2016Visual}
Tahmoresnezhad, J., Hashemi, S.:
\newblock Visual domain adaptation via transfer feature learning.
\newblock Knowledge and Information Systems (2016)  1--21

\bibitem{Long2014Transfer}
Long, M., Wang, J., Ding, G., Sun, J.:
\newblock Transfer joint matching for unsupervised domain adaptation.
\newblock In: IEEE Conference on Computer Vision and Pattern Recognition.
  (2014)  1410--1417

\bibitem{Lin2015DeepHash}
Lin, J., Morere, O., Chandrasekhar, V., Veillard, A., Goh, H.:
\newblock Deephash: Getting regularization, depth and fine-tuning right.
\newblock Mccarthy (2015)

\bibitem{Simo2015Discriminative}
Simo-Serra, E., Trulls, E., Ferraz, L., Kokkinos, I.:
\newblock Discriminative learning of deep convolutional feature point
  descriptors.
\newblock In: IEEE International Conference on Computer Vision. (2015)
  118--126

\bibitem{Sohn2016npair}
Sohn, K.:
\newblock Improved deep metric learning with multi-class n-pair loss objective.
\newblock In: Advances in Neural Information Processing Systems. (2016)
  1857--1865

\bibitem{Laurens2015Accelerating}
Laurens, V.D.M.:
\newblock Accelerating t-sne using tree-based algorithms.
\newblock Journal of Machine Learning Research \textbf{15}(1) (2015)
  3221--3245

\bibitem{Wah2011The}
Wah, C., Branson, S., Welinder, P., Perona, P., Belongie, S.:
\newblock The caltech-ucsd birds200-2011 dataset.
\newblock California Institute of Technology (2011)

\bibitem{Lecun2010The}
Lecun, Y., Cortes, C.:
\newblock The mnist database of handwritten digits.
\newblock (2010)

\bibitem{Weinberger2006Distance}
Weinberger, K.Q., Saul, L.K.:
\newblock Distance metric learning for large margin nearest neighbor
  classification.
\newblock Journal of Machine Learning Research \textbf{10}(1) (2006)  207--244

\bibitem{Liu2016Large}
Liu, W., Wen, Y.:
\newblock Large-margin softmax loss for convolutional neural networks.
\newblock In: ICML. (2016)

\bibitem{Huang2016Local}
Huang, C., Loy, C.C., Tang, X.:
\newblock Local similarity-aware deep feature embedding.
\newblock In: Advances in Neural Information Processing Systems. (2016)
  1262--1270

\bibitem{Krause20133D}
Krause, J., Stark, M., Deng, J., Fei-Fei, L.:
\newblock 3d object representations for fine-grained categorization.
\newblock In: IEEE International Conference on Computer Vision Workshops.
  (2013)  554--561

\bibitem{Nilsback08}
Nilsback, M.E., Zisserman, A.:
\newblock Automated flower classification over a large number of classes.
\newblock In: Proceedings of the Indian Conference on Computer Vision, Graphics
  and Image Processing. (Dec 2008)

\bibitem{Maji2013Fine}
Maji, S., Rahtu, E., Kannala, J., Blaschko, M., Vedaldi, A.:
\newblock Fine-grained visual classification of aircraft.
\newblock HAL - INRIA (2013)

\bibitem{Sun2014Deeply}
Sun, Y., Wang, X., Tang, X.:
\newblock Deeply learned face representations are sparse, selective, and
  robust.
\newblock In: IEEE Conference on Computer Vision and Pattern Recognition.
  (2014)  2892--2900

\bibitem{qian2015fine}
Qian, Q., Jin, R., Zhu, S., Lin, Y.:
\newblock Fine-grained visual categorization via multi-stage metric learning.
\newblock In: Proceedings of the IEEE Conference on Computer Vision and Pattern
  Recognition. (2015)  3716--3724

\bibitem{songCVPR17}
Song, H.O., Jegelka, S., Rathod, V., Murphy, K.:
\newblock Deep metric learning via facility location.
\newblock In: Computer Vision and Pattern Recognition (CVPR). (2017)

\bibitem{Yuan_2017_ICCV}
Yuan, Y., Yang, K., Zhang, C.:
\newblock Hard-aware deeply cascaded embedding.
\newblock In: The IEEE International Conference on Computer Vision (ICCV). (Oct
  2017)

\bibitem{kumar2017smart}
Kumar, V.B., Harwood, B., Carneiro, G., Reid, I., Drummond, T.:
\newblock Smart mining for deep metric learning.
\newblock arXiv preprint arXiv:1704.01285 (2017)

\bibitem{Wu_2017_ICCV}
Wu, C.Y., Manmatha, R., Smola, A.J., Krahenbuhl, P.:
\newblock Sampling matters in deep embedding learning.
\newblock In: The IEEE International Conference on Computer Vision (ICCV). (Oct
  2017)

\bibitem{wang2017deep}
Wang, J., Zhou, F., Wen, S., Liu, X., Lin, Y.:
\newblock Deep metric learning with angular loss.
\newblock arXiv preprint arXiv:1708.01682 (2017)

\bibitem{Opitz_2017_ICCV}
Opitz, M., Waltner, G., Possegger, H., Bischof, H.:
\newblock Bier - boosting independent embeddings robustly.
\newblock In: The IEEE International Conference on Computer Vision (ICCV). (Oct
  2017)

\bibitem{Movshovitz-Attias_2017_ICCV}
Movshovitz-Attias, Y., Toshev, A., Leung, T.K., Ioffe, S., Singh, S.:
\newblock No fuss distance metric learning using proxies.
\newblock In: The IEEE International Conference on Computer Vision (ICCV). (Oct
  2017)

\bibitem{wen2016discriminative}
Wen, Y., Zhang, K., Li, Z., Qiao, Y.:
\newblock A discriminative feature learning approach for deep face recognition.
\newblock In: European Conference on Computer Vision, Springer (2016)  499--515

\bibitem{Chen_2017_CVPR}
Chen, B., Deng, W., Du, J.:
\newblock Noisy softmax: Improving the generalization ability of dcnn via
  postponing the early softmax saturation.
\newblock In: The IEEE Conference on Computer Vision and Pattern Recognition
  (CVPR). (July 2017)

\bibitem{liu2017sphereface}
Liu, W., Wen, Y., Yu, Z., Li, M., Raj, B., Song, L.:
\newblock Sphereface: Deep hypersphere embedding for face recognition.
\newblock In: The IEEE Conference on Computer Vision and Pattern Recognition
  (CVPR). Volume~1. (2017)

\bibitem{Jia2014Caffe}
Jia, Yangqing, Shelhamer, Evan, Donahue, Jeff, Karayev, Sergey, Long, Jonathan:
\newblock Caffe: Convolutional architecture for fast feature embedding.
\newblock Eprint Arxiv (2014)  675--678

\end{thebibliography}
\end{document}